\documentclass[11pt,a4paper]{article}
\usepackage{hyperref}
\usepackage{naaclhlt2019}
\usepackage{times}
\usepackage{latexsym}
\usepackage{booktabs}       
\usepackage{amsfonts}       
\usepackage{nicefrac}       
\usepackage{microtype}      
\usepackage{graphicx}
\usepackage{subfigure}
\usepackage{amsmath}
\usepackage{multirow}
\usepackage{amsthm}
\usepackage{amssymb}
\usepackage{mathtools}
\usepackage{color}
\usepackage{xcolor}
\usepackage{MnSymbol}
\usepackage{makecell}
\usepackage{arydshln}
\usepackage[shortlabels]{enumitem}
\usepackage{booktabs}
\usepackage{caption}
\usepackage{wrapfig,lipsum}
\usepackage{url}
\usepackage{algorithm}
\usepackage[noend]{algpseudocode}
\usepackage{bbding}
\usepackage{pifont}
\usepackage{wasysym}

\usepackage{pifont}
\newcommand{\cmark}{\ding{51}}%
\newcommand{\xmark}{\ding{55}}%

\makeatletter
\def\BState{\State\hskip-\ALG@thistlm}
\makeatother

\DeclareMathOperator*{\argmax}{arg\,max}

\newtheorem{theorem}{Theorem}
\newtheorem{corollary}{Corollary}
\newtheorem{proposition}{Proposition}

\aclfinalcopy 
\def\aclpaperid{1632} 

\newcommand\blfootnote[1]{%
  \begingroup
  \renewcommand\thefootnote{}\footnote{#1}%
  \addtocounter{footnote}{-1}%
  \endgroup
}

\title{Strong and Simple Baselines for Multimodal Utterance Embeddings}

\author{Paul Pu Liang$^{*}$, Yao Chong Lim$^{*}$,\\ {\bf Yao-Hung Hubert Tsai, Ruslan Salakhutdinov, Louis-Philippe Morency}\\
School of Computer Science, Carnegie Mellon University\\
{\tt \{pliang,yaochonl,yaohungt,rsalakhu,morency\}@cs.cmu.edu}
}

\date{}

\begin{document}
\maketitle

\begin{abstract}
Human language is a rich multimodal signal consisting of spoken words, facial expressions, body gestures, and vocal intonations. Learning representations for these spoken utterances is a complex research problem due to the presence of multiple heterogeneous sources of information. Recent advances in multimodal learning have followed the general trend of building more complex models that utilize various attention, memory and recurrent components. In this paper, we propose two simple but strong baselines to learn embeddings of multimodal utterances. The first baseline assumes a conditional factorization of the utterance into unimodal factors. Each unimodal factor is modeled using the simple form of a likelihood function obtained via a linear transformation of the embedding. We show that the optimal embedding can be derived in closed form by taking a weighted average of the unimodal features. In order to capture richer representations, our second baseline extends the first by factorizing into unimodal, bimodal, and trimodal factors, while retaining simplicity and efficiency during learning and inference. From a set of experiments across two tasks, we show strong performance on both supervised and semi-supervised multimodal prediction, as well as significant (10 times) speedups over neural models during inference. Overall, we believe that our strong baseline models offer new benchmarking options for future research in multimodal learning.\blfootnote{$^*$authors contributed equally}
\end{abstract}

\section{Introduction}
 
Human language is a rich multimodal signal consisting of spoken words, facial expressions, body gestures, and vocal intonations~\citep{first}. At the heart of many multimodal modeling tasks lies the challenge of learning rich representations of spoken utterances from multiple modalities~\citep{papo2014language}. However, learning representations for these spoken utterances is a complex research problem due to the presence of multiple heterogeneous sources of information~\citep{baltruvsaitis2017multimodal}. This challenging yet crucial research area has real-world applications in robotics~\citep{Montalvo:2017:MRB:3029798.3038315,NODA2014721}, dialogue systems~\citep{Johnston:2002:MAM:1073083.1073146,Rudnicky2005}, intelligent tutoring systems~\citep{Li12,10.1007/978-3-642-24571-8_21,10.1007/978-3-319-91464-0_15}, and healthcare diagnosis~\citep{Wentzel:2016:FAM:2985766.2985769,Lisetti:2003:DMI:861049.861064,DBLP:journals/corr/abs-1709-01796}. Recent progress on multimodal representation learning has investigated various neural models that utilize one or more of attention, memory and recurrent components~\citep{Yang_2017_CVPR,multistage}. There has also been a general trend of building more complicated models for improved performance.

In this paper, we propose two simple but strong baselines to learn embeddings of multimodal utterances. The first baseline assumes a factorization of the utterance into unimodal factors conditioned on the joint embedding. Each unimodal factor is modeled using the simple form of a likelihood function obtained via a linear transformation of the utterance embedding. We derive a coordinate-ascent style algorithm~\citep{Wright:2015:CDA:2783158.2783189} to learn the optimal multimodal embeddings under our model. We show that, under some assumptions, maximum likelihood estimation for the utterance embedding can be derived in closed form and is equivalent to computing a weighted average of the language, visual and acoustic features. Only a few linear transformation parameters need to be learned. In order to capture bimodal and trimodal representations, our second baseline extends the first one by assuming a factorization into unimodal, bimodal, and trimodal factors~\cite{tensoremnlp17}. To summarize, our simple baselines 1) consist primarily of linear functions, 2) have few parameters, and 3) can be approximately solved in a closed form solution. As a result, they demonstrate simplicity and efficiency during learning and inference.

We perform a set of experiments across two tasks and datasets spanning multimodal personality traits recognition~\cite{Park:2014:CAP:2663204.2663260} and multimodal sentiment analysis~\citep{zadeh2016multimodal}. Our proposed baseline models 1) achieve competitive performance on supervised multimodal learning, 2) improve upon classical deep autoencoders for semi-supervised multimodal learning, and 3) are up to 10 times faster during inference. Overall, we believe that our baseline models offer new benchmarks for future multimodal research.
\section{Related Work}
\label{Related Work}

We provide a review of {\em sentence embeddings}, {\em multimodal utterance embeddings}, and {\em strong baselines}.

\subsection{Language-Based Sentence Embeddings}

Sentence embeddings are crucial for down-stream tasks such as document classification, opinion analysis, and machine translation. With the advent of deep neural networks, multiple network designs such as Recurrent Neural Networks (RNNs)~\cite{rumelhart1986learning}, Long-Short Term Memory networks (LSTMs)~\cite{Hochreiter:1997:LSM:1246443.1246450}, Temporal Convolutional Networks~\cite{bai2018empirical}, and the Transformer~\cite{NIPS2017_7181} have been proposed and achieve superior performance. However, more training data is required for larger models~\cite{peters2018deep}. In light of this challenge, researchers have started to leverage unsupervised training objectives to learn sentence embedding which showed state-of-the-art performance across multiple tasks~\cite{devlin2018bert}. In our paper, we go beyond unimodal language-based sentence embeddings and consider multimodal spoken utterances where additional information from the nonverbal behaviors is crucial to infer speaker intent.

\subsection{Multimodal Utterance Embeddings}

Learning multimodal utterance embeddings brings a new level of complexity as it requires modeling both intra-modal and inter-modal interactions~\cite{multistage}. Previous approaches have explored variants of graphical models and neural networks for multimodal data. RNNs~\citep{elman1990finding,Jain:1999:RNN:553011}, LSTMs~\cite{Hochreiter:1997:LSM:1246443.1246450}, and convolutional neural networks~\cite{NIPS2012_4824} have been extended for multimodal settings~\cite{rajagopalan2016extending,lee2018convolutional}. Experiments on more advanced networks suggested that encouraging correlation between modalities~\cite{Yang_2017_CVPR}, enforcing disentanglement on multimodal representations~\cite{factorized}, and using attention to weight modalities~\cite{gulrajani2017improved} led to better performing multimodal representations. In our paper, we present a new perspective on learning multimodal utterance embeddings by assuming a conditional factorization over the language, visual and acoustic features. Our simple but strong baseline models offer an alternative approach that is extremely fast and competitive on both supervised and semi-supervised prediction tasks.

\begin{figure*}[t!]
\centering
\vspace{-0mm}
\includegraphics[width=1.0\textwidth]{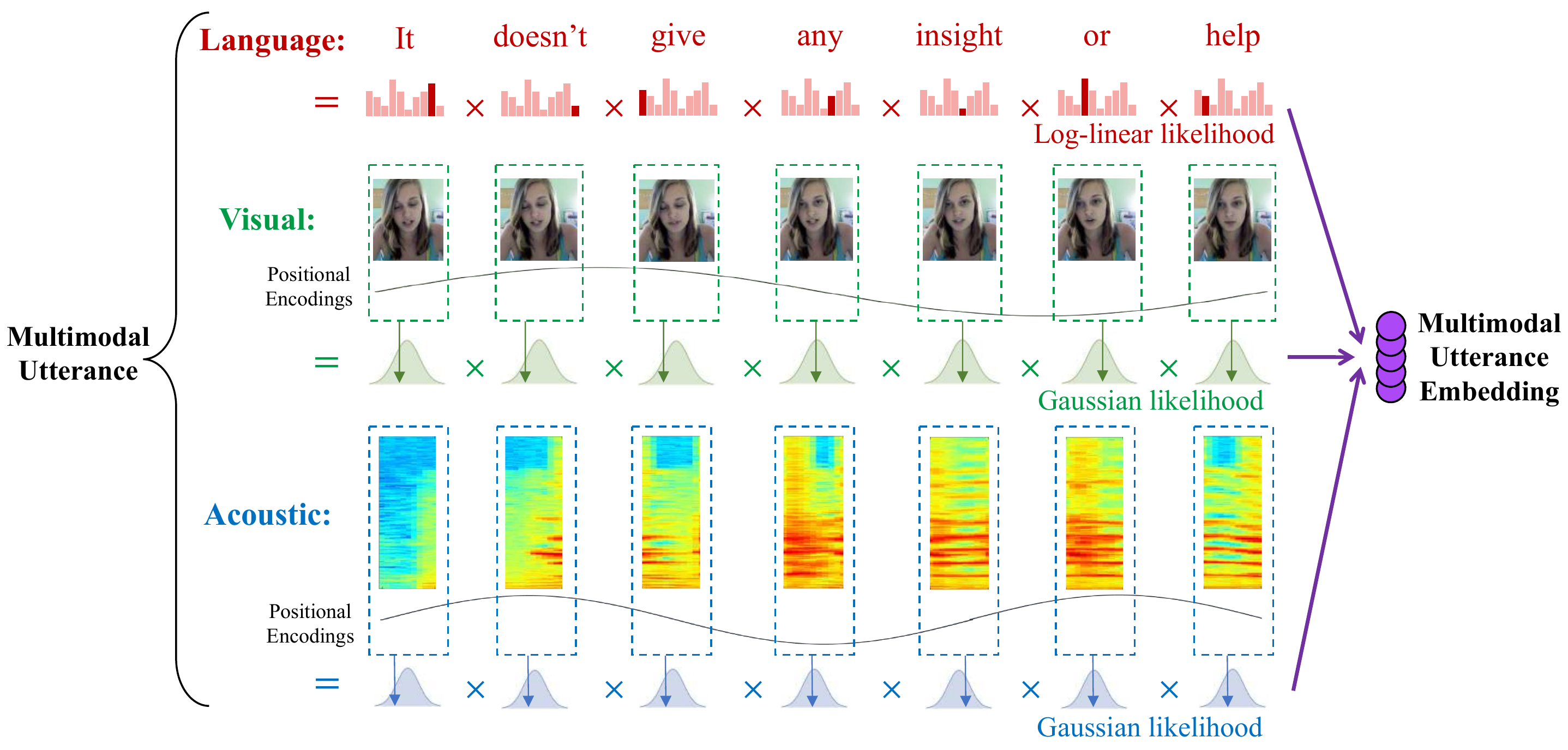}
\caption{Our baseline model assumes a factorization of the multimodal utterance into unimodal factors conditioned on the joint utterance embedding. Each unimodal factor is modeled using the simple form of a likelihood function obtained via a linear transformation of the utterance embedding. We show that, under some assumptions, maximum likelihood estimation for the utterance embedding can be derived in closed form and is equivalent to taking a weighted average of the language, visual and acoustic features.}
\vspace{-4mm}
\label{overview}
\end{figure*}

\subsection{Strong Baseline Models}

A recent trend in NLP research has been geared towards building simple but strong baselines~\cite{arora2,P18-1041,anonymous2019no,W17-3203}. The effectiveness of these baselines indicate that complicated network components are not always required. For example,~\citet{arora2} constructed sentence embeddings from weighted combinations of word embeddings which requires no trainable parameters yet generalizes well to down-stream tasks. \citet{P18-1041} proposed parameter-free pooling operations on word embeddings for document classification, text sequence matching, and text tagging. \citet{anonymous2019no} discovered that random sentence encoders achieve competitive performance as compared to larger models that involve expensive training and tuning. \citet{W17-3203} emphasized the importance of choosing a basic neural machine translation model and carefully reporting the relative gains achieved by the proposed techniques. Authors in other domains have also highlighted the importance of developing strong baselines~\cite{lakshminarayanan2017simple,sharif2014cnn}. To the best of our knowledge, our paper is the first to propose and evaluate strong, non-neural baselines for multimodal utterance embeddings.
\section{Baselines for Multimodal Learning}

\subsection{Notation}

Suppose we are given video data where each utterance segment is denoted as $\mathbf{s}$. Each segment contains individual words $w$ in a sequence $\mathbf{w}$, visual features $v$ in a sequence $\mathbf{v}$, and acoustic features $a$ in a sequence $\mathbf{a}$. We aim to learn a representation $m_{\mathbf{s}}$ for each segment that summarizes information present in the multimodal utterance.

\subsection{Background}

Our model is related to the work done by~\citet{arora2016latent} and~\citet{arora2}. In the following, we first provide a brief review of their method. Given a sentence,~\citet{arora2016latent} aims to learn a sentence embedding $c_s$. They do so by assuming that the probability of observing a word $w_t$ at time $t$ is given by a log-linear word production model~\citep{Mnih:2007:TNG:1273496.1273577} with respect to $c_s$:
\begin{equation}
\label{initial}
\mathbb{P} [\textrm{$w_t$} | c_s] = \frac{\exp \left( \langle v_{w_t}, c_s \rangle \right)}{Z_{c_s}},
\end{equation}
where $c_s$ is the sentence embedding (context), $v_{w_t}$ represents the word vector associated with word $w_t$ and $Z_{c_s} = \sum_{w \in V} \exp \left( \langle v_{w}, c_s \rangle \right)$ is a normalizing constant over all words in the vocabulary. Given this posterior probability, the desired sentence embedding $c_s$ can be obtained by maximizing Equation~\eqref{initial} with respect to $c_s$. Under some assumptions on $c_s$, this maximization yields a closed-form solution which provides an efficient learning algorithm for sentence embeddings.

\citet{arora2} further extends this model by introducing a ``smoothing term'' $\alpha$ to account for the production of frequent stop words or out of context words independent of the discourse vector. Given estimated unigram probabilities $p(w)$, the probability of a word at time $t$ is given by
\begin{equation}
\label{text}
\mathbb{P} [\textrm{$w_t$} | c_s] = \alpha p(w_t) + (1-\alpha) \frac{\exp \left( \langle v_{w_t}, c_s \rangle \right)}{Z_{c_s}}.
\end{equation}
Under this model with the additional hyperparameter $\alpha$, we can still obtain a closed-form solution for the optimal $c_s$.

\subsection{Baseline 1: Factorized Unimodal Model}
\label{model}

In this subsection, we outline our method for learning representations of multimodal utterances. An overview of our proposed baseline model is shown in Figure~\ref{overview}. Our method begins by assuming a factorization of the multimodal utterance into unimodal factors conditioned on the joint utterance embedding. Next, each unimodal factor is modeled using the simple form of a likelihood function obtained via a linear transformation of the utterance embedding. Finally, we incorporate positional encodings to represent temporal information in the features. We first present the details of our proposed baseline before deriving a coordinate ascent style optimization algorithm to learn utterance embeddings in our model.

\noindent \textbf{Unimodal Factorization:} We use $m_{\mathbf{s}}$ to represent the multimodal utterance embedding. To begin, we simplify the composition of $m_{\mathbf{s}}$ by assuming that the segment $\mathbf{s}$ can be conditionally factorized into words ($\mathbf{w}$), visual features ($\mathbf{v}$), and acoustic features ($\mathbf{a}$). Each factor is also associated with a temperature hyperparameter ($\alpha_{\mathbf{w}}$, $\alpha_{\mathbf{v}}$, $\alpha_{\mathbf{a}}$) that represents the contribution of each factor towards the multimodal utterance. The likelihood of a segment $\mathbf{s}$ given the embedding $m_{\mathbf{s}}$ is therefore
\begin{align}
\label{factorize}
\nonumber &\mathbb{P} [\mathbf{s}|m_{\mathbf{s}}] = \mathbb{P} [\mathbf{w}|m_{\mathbf{s}}]^{\alpha_{\mathbf{w}}} \left(\mathbb{P} [\mathbf{v}|m_{\mathbf{s}}]\right)^{\alpha_{\mathbf{v}}} \mathbb{P} [\mathbf{a}|m_{\mathbf{s}}]^{\alpha_{\mathbf{a}}} \\
&= \prod_{w \in \mathbf{w}} \mathbb{P} [w|m_{\mathbf{s}}]^{\alpha_{\mathbf{w}}} \prod_{v \in \mathbf{v}} \mathbb{P} [v|m_{\mathbf{s}}]^{\alpha_{\mathbf{v}}} \prod_{a \in \mathbf{a}} \mathbb{P} [a|m_{\mathbf{s}}]^{\alpha_{\mathbf{a}}}.
\end{align}

\vspace{-0mm}
\noindent \textbf{Choice of Likelihood Functions:} As suggested by~\citet{arora2}, given $m_{\mathbf{s}}$, we model the probability of a word $w$ using Equation~\eqref{text}. In order to analytically solve for ${m_{\mathbf{s}}}$, a lemma is introduced by~\citet{arora2016latent,arora2} which states that the partition function $Z_{{m_{\mathbf{s}}}}$ is concentrated around some constant $Z$ (for all $m_{\mathbf{s}}$). This lemma is also known as the ``self-normalizing'' phenomenon of log-linear models~\citep{N15-1027,Andreas:2015:ASL:2969239.2969438}. We use the same assumption and treat $Z_{{m_{\mathbf{s}}}_t} \approx Z$ for all $m_{\mathbf{s}}$.

Unlike discrete text tokens, the visual features are continuous. We assume that the visual features are generated from an isotropic Gaussian distribution. In section~\ref{gaussian_vis}, we visually analyze the distribution of the features for real world datasets and show that these likelihood modeling assumptions are indeed justified. The Gaussian distribution is parametrized by simple linear transformations $W_v^\mu, W_v^\sigma \in \mathbb{R}^{|v| \times |m_{\mathbf{s}}|}$ and $b_v^\mu, b_v^\sigma \in \mathbb{R}^{|v|}$:
\begin{align}
\label{gaussian}
v | m_{\mathbf{s}} &\sim \mathcal{N}(\mu_v, \sigma_v^2), \\
\mu_v &= W_v^\mu m_{\mathbf{s}} + b_v^\mu, \\
\sigma_v &= \textrm{diag} \left( \exp \left( W_v^\sigma m_{\mathbf{s}} + b_v^\sigma \right) \right).
\end{align}
Similarly, we also assume that the continuous acoustic features are generated from a different isotropic Gaussian distribution parametrized as:
\begin{align}
a | m_{\mathbf{s}} &\sim \mathcal{N}(\mu_a, \sigma_a^2), \\
\mu_a &= W_a^\mu m_{\mathbf{s}} + b_a^\mu, \\
\sigma_a &= \textrm{diag} \left( \exp \left( W_a^\sigma m_{\mathbf{s}} + b_a^\sigma \right) \right).
\end{align}

\noindent \textbf{Positional Encodings:} Finally, we incorporate positional encodings~\cite{NIPS2017_7181} into the features to represent temporal information. We use $d$-dimensional positional encodings with entries:
\begin{align}
PE_{pos,2i} &= \sin \left( pos/10000^{2i/d} \right), \\
PE_{pos,2i+1} &= \cos \left( pos/10000^{2i/d} \right).
\end{align}
where $pos$ is the position (time step) and $i \in [1,d]$ indexes the dimension of the positional encodings. We call this resulting model Multimodal Baseline 1 (\textbf{MMB1}).

\subsection{Optimization for Baseline 1}
\label{optimization}

We define our objective function by the log-likelihood of the observed multimodal utterance $\mathbf{s}$. The maximum likelihood estimator of the utterance embedding $m_{\mathbf{s}}$ and the linear transformation parameters $W$ and $b$ can then be obtained by maximizing this objective
\begin{align}
\label{objective}
\mathcal{L} (m_{\mathbf{s}}, W^{}_{}, b^{}_{}; \mathbf{s}) &= \log \mathbb{P} [\mathbf{s}|m_{\mathbf{s}}; W^{}_{}, b^{}_{}],
\end{align}
where we use $W$ and $b$ to denote all linear transformation parameters.

\noindent \textbf{Coordinate Ascent Style Algorithm:}
Since the objective~\eqref{objective} is not jointly convex in $m_{\mathbf{s}}$, $W$ and $b$, we optimize by alternating between: 1) solving for $m_{\mathbf{s}}$ given the parameters $W$ and $b$ at the current iterate, and 2) given $m_{\mathbf{s}}$, updating $W$ and $b$ using a gradient-based algorithm. This resembles the coordinate ascent optimization algorithm which maximizes the objective according to each coordinate separately~\citep{Tseng:2001:CBC:565614.565615,Wright:2015:CDA:2783158.2783189}. Algorithm~\ref{sb1} presents our method for learning utterance embeddings. In the following sections, we describe how to solve for $m_{\mathbf{s}}$ and update $W$ and $b$.

\begin{algorithm}[t!]
\caption{Baseline 1}
\label{sb1}
\begin{algorithmic}[1]
\Procedure{Baseline 1}{}
\State Initialize $m_{\mathbf{s}}, W^{}_{}, b^{}_{}$.
\For{each iteration}
\State Fix $W^{(k)},b^{(k)}$, compute $m_{\mathbf{s}}^{(k)}$ by~\eqref{mle_m}.
\State Fix $m_{\mathbf{s}}^{(k)}$, compute $\nabla_W \mathcal{L}$ by (\ref{grad_W1}-\ref{grad_W2}).
\State Fix $m_{\mathbf{s}}^{(k)}$, compute $\nabla_b \mathcal{L}$ by (\ref{grad_b1}-\ref{grad_b2}).
\State Update $W^{(k+1)}$ from $W^{(k)}$ and $\nabla_W \mathcal{L}$.
\State Update $b^{(k+1)}$ from $b^{(k)}$ and $\nabla_b \mathcal{L}$.
\EndFor
\EndProcedure
\end{algorithmic}
\end{algorithm}

\noindent \textbf{Solving for $m_{\mathbf{s}}$:} We first derive an algorithm to solve for the optimal $m_{\mathbf{s}}$ given the log likelihood objective in~\eqref{objective}, and parameters $W$ and $b$.
\begin{theorem}
\label{mle}
[Solving for $m_{\mathbf{s}}$]  Assume the optimal $m_{\mathbf{s}}$ lies on the unit sphere (i.e. $\| m_{\mathbf{s}} \|_2^2 = 1$), then closed form of $m_{\mathbf{s}}$ in line 4 in Algorithm~\ref{sb1} is
\vspace{-0mm}
\begin{align}
\label{mle_m}
\nonumber m_{\mathbf{s}}^* &= \sum_{w \in s} \psi_w w \\
\nonumber &+ \sum_{v \in s} \left( W_v^{\mu\top} \tilde{v}^{(1)} \psi_v^{(1)} + W_v^{\sigma\top} \tilde{v}^{(2)} \psi_v^{(2)} \right)\\
&+ \sum_{a \in s} \left( W_a^{\mu\top} \tilde{a}^{(1)} \psi_a^{(1)} + W_a^{\sigma\top} \tilde{a}^{(2)} \psi_a^{(2)} \right).
\end{align}
\vspace{-0mm}
where the shifted visual and acoustic features are:
\begin{align}
\tilde{v}^{(1)} = v-b_v^\mu, \ \tilde{v}^{(2)} = (v-b_v^\mu) \otimes (v-b_v^\mu), \\
\tilde{a}^{(1)} = a-b_a^\mu, \ \tilde{a}^{(2)} = (a-b_a^\mu) \otimes (a-b_a^\mu),
\vspace{-0mm}
\end{align}
where $\otimes$ denotes Hadamard (element-wise) product and the weights $\psi$'s are given as follows:
\vspace{-0mm}
\begin{align}
\psi_w &= \frac{{\alpha_\mathbf{w}} (1-\alpha)/(\alpha Z)}{p(w)+(1-\alpha)/(\alpha Z)}, \\
\psi_v^{(1)} &= \textup{diag} \left( \frac{\alpha_{\mathbf{v}}}{\exp \left(2 b_v^\sigma \right)} \right), \\
\psi_v^{(2)} &= \textup{diag} \left(\frac{\alpha_{\mathbf{v}}}{\exp \left(2b^\sigma_v \right)}-\alpha_{\mathbf{v}}\right), \\
\psi_a^{(1)} &= \textup{diag} \left( \frac{\alpha_{\mathbf{a}}}{\exp \left(2 b_a^\sigma \right)} \right), \\
\psi_a^{(2)} &= \textup{diag} \left(\frac{\alpha_{\mathbf{a}}}{\exp \left(2b^\sigma_a \right)}-\alpha_{\mathbf{a}}\right).
\end{align}
\vspace{-0mm}
\end{theorem}

\begin{proof}
The proof is adapted from~\citet{arora2} and involves computing the gradients $\nabla_{m_{\mathbf{s}}} \log \mathbb{P} [\cdot|m_{\mathbf{s}}]^{\alpha_{\mathbf{\cdot}}}$. We express $\log \mathbb{P} [\cdot|m_{\mathbf{s}}]$ via a Taylor expansion approximation and we observe that $\log \mathbb{P} [\cdot|m_{\mathbf{s}}] \approx c + \langle m_{\mathbf{s}}, g \rangle$ for a constant $c$ and a vector $g$. Then, we can obtain $m_{\mathbf{s}}^*$ by computing $\argmax_{m_{\mathbf{s}}^*} \mathcal{L} (m_{\mathbf{s}}, W^{}_{}, b^{}_{}; \mathbf{s})$ which yields a closed-form solution. Please refer to the supplementary material for proof details.
\end{proof}

Observe that the optimal embedding $m_{\mathbf{s}}^*$ is a weighted average of the word features $w$ and the (shifted and transformed) visual and acoustic features, $\tilde{v}$ and $\tilde{a}$. Our choice of a Gaussian likelihood for the visual and acoustic features introduces a squared term $(v-b_v^\mu) \otimes (v-b_v^\mu)$ to account for the $\ell_2$ distance present in the pdf. The transformation matrix $W^\top$ transforms the visual and acoustic features into the multimodal embedding space. Regarding the weights $\psi$, note that: 1) the weights are proportional to the global temperatures $\alpha$ assigned to that modality, 2) the weights $\psi_w$ are inversely proportional to $p(w)$ (rare words carry more weight), and 3) the weights $\psi_v$ and $\psi_a$ scale each feature dimension inversely by their magnitude.

\noindent \textbf{Updating $W$ and $b$:} To find the optimal linear transformation parameters $W$ and $b$ to maximize the objective in~\eqref{objective}, we perform gradient-based optimization on $W$ and $b$ (in Algorithm~\ref{sb1} line 5-8).

\begin{proposition}
\label{gradient}
[Updating $W$ and $b$] The gradients $\nabla_W \mathcal{L} (m_{\mathbf{s}}, W^{}_{}, b^{}_{})$ and $\nabla_b \mathcal{L} (m_{\mathbf{s}}, W^{}_{}, b^{}_{})$, in each dimension, are:
{\fontsize{10}{12}\selectfont
\begin{align}
\label{grad_W1}
\nabla_{{W^\mu_v}_{ij}} \mathcal{L} (m_{\mathbf{s}}, W^{}_{}, b^{}_{}) &= \alpha_{\mathbf{v}} \textup{tr} \left[ \left(\sigma_v^{-2} (v-\mu_v) \right)^\top {m_{\mathbf{s}}}_j \right],
\end{align}
\vspace{-0mm}
\begin{align}
\label{grad_W2}
\nonumber & \nabla_{{W^\sigma_v}_{ij}} \mathcal{L} (m_{\mathbf{s}}, W^{}_{}, b^{}_{}) = \\
&- \frac{\alpha_{\mathbf{v}}}{2} \textup{tr} \left[ \left( \sigma_v^{-2} - \sigma_v^{-2} (v-\mu_v) (v-\mu_v)^\top \sigma_v^{-2} \right)^\top {\sigma_v}_{ii} {m_{\mathbf{s}}}_j \right],
\end{align}
\vspace{-0mm}
\begin{align}
\label{grad_b1}
\small
\nabla_{{b^\mu_v}_{i}} \mathcal{L} (m_{\mathbf{s}}, W^{}_{}, b^{}_{}) &= \alpha_{\mathbf{v}} \textup{tr} \left[ \left(\sigma_v^{-2} (v-\mu_v) \right)^\top \right],
\end{align}
\vspace{-0mm}
\begin{align}
\label{grad_b2}
\small
\nonumber & \nabla_{{b^\sigma_v}_{i}} \mathcal{L} (m_{\mathbf{s}}, W^{}_{}, b^{}_{}) \\
&= - \frac{\alpha_{\mathbf{v}}}{2} \textup{tr} \left[ \left( \sigma_v^{-2} - \sigma_v^{-2} (v-\mu_v) (v-\mu_v)^\top \sigma_v^{-2} \right)^\top {\sigma_v}_{ii} \right].
\end{align}
}
\end{proposition}

\begin{proof}
The proof involves differentiating the log likelihood of a multivariate Gaussian with respect to $\mu$ and $\sigma$ before applying the chain rule to $\mu = W^\mu m_{\mathbf{s}} + b^\mu$ and $\sigma = \textrm{diag} \left( \exp \left( W^\sigma m_{\mathbf{s}} + b^\sigma \right) \right)$.
\end{proof}

\subsection{Baseline 2: Incorporating Bimodal and Trimodal Interactions}

So far, we have assumed the utterance segment $s$ can be independently factorized into unimodal features. In this subsection, we extend the setting to take account for bimodal and trimodal interactions. We adopt the idea of early-fusion~\citep{srivastava2012multimodal}, which means the bimodal and trimodal interactions are captured by the concatenated features from different modalities. Specifically, we define our factorized model as:
\begin{align}
\nonumber &\mathbb{P} [\mathbf{s}|m_{\mathbf{s}}] = \mathbb{P} [\mathbf{w}|m_{\mathbf{s}}]^{\alpha_{\mathbf{w}}} \mathbb{P} [\mathbf{v}|m_{\mathbf{s}}]^{\alpha_{\mathbf{v}}} \mathbb{P} [\mathbf{a}|m_{\mathbf{s}}]^{\alpha_{\mathbf{a}}} \\
\nonumber & \mathbb{P} [\left(\mathbf{w} \oplus \mathbf{v}\right)|m_{\mathbf{s}}]^{\alpha_{\mathbf{wv}}} \mathbb{P} [\left(\mathbf{w} \oplus \mathbf{a}\right)|m_{\mathbf{s}}]^{\alpha_{\mathbf{wa}}} \\
& \mathbb{P} [\left(\mathbf{v} \oplus \mathbf{a} \right)|m_{\mathbf{s}}]^{\alpha_{\mathbf{va}}} \mathbb{P} [\left(\mathbf{w} \oplus \mathbf{v} \oplus \mathbf{a} \right)|m_{\mathbf{s}}]^{\alpha_{\mathbf{wva}}},
\end{align}
where $\oplus$ denotes vector concatenation for bimodal and trimodal features. Each of the individual probabilities factorize in the same way as Equation~\eqref{factorize} (i.e. $ \mathbb{P} [\mathbf{a}|m_{\mathbf{s}}]^{\alpha_{\mathbf{a}}} = \prod_{a \in \mathbf{a}} \mathbb{P} [a|m_{\mathbf{s}}]^{\alpha_{\mathbf{a}}}$). Similar to baseline 1, we assume a log-linear likelihood~\eqref{text} for $\mathbb{P} [{w}|m_{\mathbf{s}}]$ and a Gaussian likelihood~\eqref{gaussian} for all remaining terms. We call this Multimodal Baseline 2 (\textbf{MMB2}).

\subsection{Optimization for Baseline 2}

The optimization algorithm derived in section~\ref{optimization} can be easily extended to learn $m_{\mathbf{s}}$, $W$ and $b$ in Baseline 2. We again alternate between the 2 steps of 1) solving for $m_{\mathbf{s}}$ given the parameters $W$ and $b$ at the current iterate, and 2) given $m_{\mathbf{s}}$, updating $W$ and $b$ using a gradient-based algorithm. 

\noindent \textbf{Solving for $m_{\mathbf{s}}$:} We state a result that derives the closed-form of $m_{\mathbf{s}}$ given $W$ and $b$:
\begin{corollary}
\label{mle2}
[Solving for $m_{\mathbf{s}}$] Assume that the optimal $m_{\mathbf{s}}$ lies on the unit sphere (i.e. $\| m_{\mathbf{s}} \|_2^2 = 1$). The closed-form (in Algorithm~\ref{sb1} line 4) for $m_{\mathbf{s}}$ is:
\begin{align}
& \nonumber m_{\mathbf{s}}^* = \sum_{w \in \mathbf{w}} \psi_w w \\
\nonumber &+ \sum_{v \in \mathbf{v}} \left( W_v^{\mu\top} \tilde{v}^{(1)} \psi_v^{(1)} + W_v^{\sigma\top} \tilde{v}^{(2)} \psi_v^{(2)} \right) \\
\nonumber &+ \sum_{a \in \mathbf{a}} \left( W_a^{\mu\top} \tilde{a}^{(1)} \psi_a^{(1)} + W_a^{\sigma\top} \tilde{a}^{(2)} \psi_a^{(2)} \right) \\
&+ \sum_{\substack{\mathbf{f} \in \{ \mathbf{w} \oplus \mathbf{v}, \mathbf{w} \oplus \mathbf{a}, \\ \mathbf{v} \oplus \mathbf{a}, \mathbf{w} \oplus \mathbf{v} \oplus \mathbf{a} \} }} \sum_{f \in \mathbf{f}} \left( W_f^{\mu\top} \tilde{f}^{(1)} \psi_{f}^{(1)} + W_f^{\sigma\top} \tilde{f}^{(2)} \psi_{f}^{(2)} \right)
\end{align}
where the shifted (and squared) visual features are:
\begin{align}
\tilde{v}^{(1)} = v-b_v^\mu, \ \tilde{v}^{(2)} = (v-b_v^\mu) \otimes (v-b_v^\mu),
\end{align}
(and analogously for $\tilde{f}^{(1)}, \tilde{f}^{(2)}, f \in \{ a,w\oplus v,w\oplus a,v\oplus a,w\oplus v\oplus a\}$). The weights $\psi$'s are:
\begin{align}
\psi_w &= \frac{{\alpha_\mathbf{w}} (1-\alpha)/(\alpha Z)}{p(w)+(1-\alpha)/(\alpha Z)}, \\
\psi_v^{(1)} &= \textup{diag} \left( \frac{\alpha_{\mathbf{v}}}{\exp \left(2 b_v^\sigma \right)} \right), \\
\psi_v^{(2)} &= \textup{diag} \left(\frac{\alpha_{\mathbf{v}}}{\exp \left(2b^\sigma_v \right)}-\alpha_{\mathbf{v}}\right).
\end{align}
(and analogously for $\psi_{f}^{(1)}, \psi_{f}^{(2)}, f \in \{ a,w\oplus v,w\oplus a,v\oplus a,w\oplus v\oplus a\}$).
\end{corollary}

\begin{proof}
The proof is a symmetric extension of Theorem~\ref{mle} to take into account the Gaussian likelihoods for bimodal and trimodal features.
\end{proof}

\noindent \textbf{Updating $W$ and $b$:} The gradient equations for updating $W$ and $b$ are identical to those derived in Proposition~\ref{gradient}, Equations (\ref{grad_W1}-\ref{grad_b2}).

\subsection{Multimodal Prediction}

Given the optimal embeddings $m_{\mathbf{s}}$, we can now train a classifier from $m_{\mathbf{s}}$ to labels $y$ for multimodal prediction. $m_{\mathbf{s}}$ can also be fine-tuned on labeled data (i.e. taking gradient descent steps to update $m_{\mathbf{s}}$ with respect to the task-specific loss functions) to learn task-specific multimodal utterance representations. In our experiments, we use a fully connected neural network for our classifier.
\section{Experimental Setup}

To evaluate the generalization of our models, we perform experiments on multimodal speaker traits recognition and multimodal sentiment analysis. The code for our experiments is released at \url{https://github.com/yaochie/multimodal-baselines}, and all datasets for our experiments can be downloaded at \url{https://github.com/A2Zadeh/CMU-MultimodalSDK}.

\subsection{Datasets}

All datasets consist of monologue videos where the speaker's intentions are conveyed through the language, visual and acoustic modalities. The multimodal features are described in the next subsection.

\noindent \textbf{Multimodal Speaker Traits Recognition} involves recognizing speaker traits based on multimodal utterances. \textbf{POM}~\cite{Park:2014:CAP:2663204.2663260} contains 903 videos each annotated for speaker traits: confident (con), voice pleasant (voi), dominance (dom), vivid (viv), reserved (res), trusting (tru), relaxed (rel), outgoing (out), thorough (tho), nervous (ner), and humorous (hum). The abbreviations (inside parentheses) are used in the tables.

\begin{figure}[t!]
\centering
\includegraphics[width=1.05\linewidth]{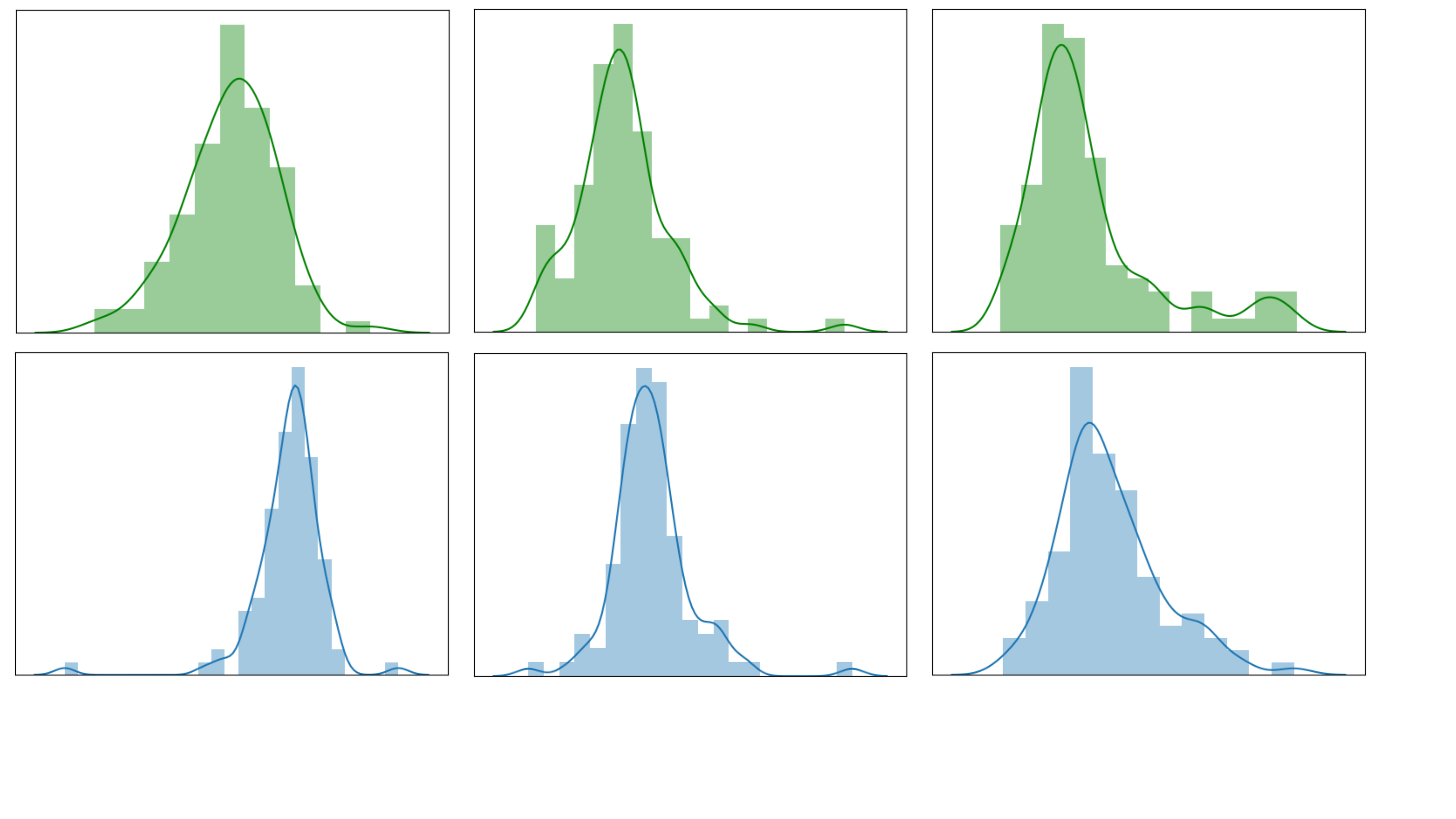}
\vspace{-10mm}
\caption{Histogram visualizations of the visual (top) and acoustic (bottom) features in some CMU-MOSI multimodal utterances. Many of the features converge to a Gaussian distribution across the time steps in the utterance, justifying our parametrization for the visual and acoustic likelihood functions.}
\vspace{-4mm}
\label{vis}
\end{figure}

\newcolumntype{K}[1]{>{\centering\arraybackslash}p{#1}}

\begin{table*}[!htb]
\begin{minipage}{.75\linewidth}
\fontsize{7}{10}\selectfont
\setlength\tabcolsep{3.3pt}
\begin{tabular}{l : *{16}{K{0.7cm}}}
\Xhline{3\arrayrulewidth}
Dataset & \multicolumn{16}{c}{\textbf{POM Personality Trait Recognition, measured in} MAE} \\
Task & \multicolumn{1}{c}{Con} & \multicolumn{1}{c}{Voi} & \multicolumn{1}{c}{Dom} & \multicolumn{1}{c}{Viv} & \multicolumn{1}{c}{Res} & \multicolumn{1}{c}{Tru} & \multicolumn{1}{c}{Rel} & \multicolumn{1}{c}{Out} & \multicolumn{1}{c}{Tho} & \multicolumn{1}{c}{Ner} & \multicolumn{1}{c}{Hum}\\
\Xhline{0.5\arrayrulewidth}
Majority & 1.483&1.089&1.167&1.158&1.166&0.743&0.753&0.872&0.939&1.181&1.774 \\
SVM	& 1.071&0.938&0.865&1.043&0.877&0.536&0.594&0.702&0.728&0.714&0.801\\
DF	& 1.033&0.899&0.870&0.997&0.884&0.534&0.591&0.698&0.732&0.695&0.768 \\
EF-LSTM$^{(\star)}$ & 1.035&0.911&0.880&0.981&0.872&0.556&0.594&0.700&0.712&0.706&0.762\\
MV-LSTM	& 1.029&0.971&0.944&0.976&0.877&0.523&0.625&0.703&0.792&\textbf{0.687}&0.770 \\
BC-LSTM	& 1.016&0.914&\textbf{0.859}&\textbf{0.905}&0.888&0.564&0.630&0.708&\textbf{0.680}&0.705&0.767	 \\
TFN  & 1.049&0.927&0.864&1.000&0.900&0.572&0.621&0.706&0.743&0.727&0.770\\
MFN  & \textbf{0.952}&\textbf{0.882}&\textbf{0.835}&\textbf{0.908}&\textbf{0.821}&\textbf{0.521}&\textbf{0.566}&\textbf{0.679}&\textbf{0.665}&\textbf{0.654}&\textbf{0.727} \\
\Xhline{0.5\arrayrulewidth}
\textbf{MMB2} & \textbf{1.015} & \textbf{0.878} & 0.885 & 0.967 & \textbf{0.857} & \textbf{0.522} & \textbf{0.578} & \textbf{0.685} & 0.705 & 0.692 & \textbf{0.726} \\
\Xhline{3\arrayrulewidth}
\end{tabular}

\vspace{2mm}

\fontsize{7}{10}\selectfont
\setlength\tabcolsep{3.3pt}
\begin{tabular}{l : *{16}{K{0.7cm}}}
\Xhline{3\arrayrulewidth}
Dataset & \multicolumn{16}{c}{\textbf{POM Personality Trait Recognition, measured in $r$}} \\
Task & \multicolumn{1}{c}{Con} & \multicolumn{1}{c}{Voi} & \multicolumn{1}{c}{Dom} & \multicolumn{1}{c}{Viv} & \multicolumn{1}{c}{Res} & \multicolumn{1}{c}{Tru} & \multicolumn{1}{c}{Rel} & \multicolumn{1}{c}{Out} & \multicolumn{1}{c}{Tho} & \multicolumn{1}{c}{Ner} & \multicolumn{1}{c}{Hum}\\
\Xhline{0.5\arrayrulewidth}
Majority & -0.041&-0.104&-0.031&-0.044&0.006&-0.077&-0.024&-0.085&-0.130&0.097&-0.069 \\
SVM	& 0.063&-0.004&0.141&0.076&0.134&0.168&0.104&0.066&0.134&0.068&0.147\\
DF	& 0.240&0.017&0.139&0.173&0.118&0.143&0.019&0.093&0.041&0.136&0.259 \\
EF-LSTM$^{(\star)}$ & 0.221&0.042&0.151&0.239&0.268&0.069&0.092&0.215&0.252&0.159&0.272 \\
MV-LSTM	& 0.358&0.131&0.146&0.347&0.323&\textbf{0.237}&0.119&0.238&0.284&0.258&0.317 \\
BC-LSTM		& \textbf{0.359}&0.081&0.234&0.417&0.450&0.109&0.075&0.078&0.363&0.184&0.319\\
TFN &  0.089&0.030&0.020&0.204&-0.051&-0.064&0.114&0.060&0.048&-0.002&0.213 \\
MFN  &  \textbf{0.395}&\textbf{0.193}&\textbf{0.313}&\textbf{0.431}&\textbf{0.333}&\textbf{0.296}&\textbf{0.255}&\textbf{0.259}&\textbf{0.381}&\textbf{0.318}&\textbf{0.386}\\
\Xhline{0.5\arrayrulewidth}
\textbf{MMB2} & 0.350 & \textbf{0.220} & \textbf{0.333} & \textbf{0.434} & \textbf{0.332} & 0.176 & \textbf{0.224} & \textbf{0.318} & \textbf{0.394} & \textbf{0.296} & \textbf{0.366} \\
\Xhline{3\arrayrulewidth}
\end{tabular}
\end{minipage}
\begin{minipage}{.1\linewidth}
\fontsize{7}{10}\selectfont
\vspace{-14.0mm}
\setlength\tabcolsep{1.1pt}
\begin{tabular}{l : *{16}{K{0.9cm}}}
\Xhline{3\arrayrulewidth}
Dataset & \multicolumn{2}{c}{\textbf{CMU-MOSI}} \\
Task & \multicolumn{2}{c}{\textbf{Sentiment}} \\
Metric       & A$({2})$ & F1 \\
\Xhline{0.5\arrayrulewidth}
Majority       & 50.2 &   50.1   \\
RF         & 56.4 &  56.3   \\
THMM		& 50.7	& 45.4	\\
EF-HCRF$^{(\star)}$	& 65.3 & 65.4 \\
MV-HCRF$^{(\star)}$	& 65.6 & 65.7 \\
SVM-MD     & 71.6 &   72.3   \\
C-MKL          & 72.3 &   72.0   \\
DF              & 72.3 &   72.1   \\
SAL-CNN    & 73.0 &   72.6     \\
EF-LSTM$^{(\star)}$		& 74.3 &   74.3  \\
MV-LSTM			& 73.9 &   74.0   \\
BC-LSTM         & 73.9 &   73.9   \\
TFN             & 74.6 &   74.5   \\
MFN & \textbf{77.4}	 & \textbf{77.3} \\
\Xhline{0.5\arrayrulewidth}
\textbf{MMB1}     & 73.6 & 73.4 \\
\textbf{MMB2}     & \textbf{75.2} & \textbf{75.1} \\
\Xhline{3\arrayrulewidth}
\end{tabular}
\end{minipage}
\caption{Results for multimodal personality trait recognition on POM (left) and multimodal sentiment analysis on CMU-MOSI (right). EF-LSTM$^{(\star)}$and HCRF$^{(\star)}$ denote the best result obtained from the LSTM and HCRF variants respectively. The top two results are highlighted in bold. Our proposed baseline model (MMB2), despite its simplicity, often ranks in the top two models and outperforms many large neural models such as C-MKL, DF, SAL-CNN, EF-LSTM, MV-LSTM, BC-LSTM, TFN, and MFN.}
\label{all}
\end{table*}

\noindent \textbf{Multimodal Sentiment Analysis} involves analyzing speaker sentiment based on video content. Multimodal sentiment analysis extends conventional language-based sentiment analysis to a multimodal setup where both verbal and non-verbal signals contribute to the expression of sentiment. We use \textbf{CMU-MOSI}~\cite{zadeh2016multimodal} which consists of 2199 opinion segments from online videos each annotated with sentiment from strongly negative $(-3)$ to strongly positive $(+3)$.

\subsection{Multimodal Features and Alignment}
GloVe word embeddings~\cite{pennington2014glove}, Facet~\cite{emotient} and COVAREP~\cite{degottex2014covarep} are extracted for the language, visual and acoustic modalities respectively\footnote{Details on feature extraction are in supplementary.}. Forced alignment is performed using P2FA~\cite{P2FA} to obtain the exact utterance times of each word. The video and audio features are aligned by computing the expectation of their features over each word interval~\cite{multistage}.

\subsection{Evaluation Metrics}
For classification, we report multiclass classification accuracy A$({c})$ where $c$ denotes the number of classes and F1 score. For regression, we report Mean Absolute Error (MAE) and Pearson's correlation ($r$). For MAE lower values indicate better performance. For all remaining metrics, higher values indicate better performance.

\section{Results and Discussion}

\subsection{Gaussian Likelihood Assumption}
\label{gaussian_vis}

Before proceeding to the experimental results, we perform some sanity checks on our modeling assumptions. We plotted histograms of the visual and acoustic features in CMU-MOSI utterances to visually determine if they resemble a Gaussian distribution. From the plots in Figure~\ref{vis}, we observe that many of the features indeed converge approximately to a Gaussian distribution across the time steps in the utterance, justifying the parametrization for the visual and acoustic likelihood functions in our model.

\subsection{Supervised Learning}

Our first set of experiments evaluates the performance of our baselines on two multimodal prediction tasks: multimodal sentiment analysis on CMU-MOSI and multimodal speaker traits recognition on POM. On CMU-MOSI (right side of Table~\ref{all}), our model \textbf{MMB2} performs competitively against many neural models including early fusion deep neural networks~\cite{nojavanasghari2016deep}, several variants of LSTMs (stacked, bidirectional etc.)~\citep{Hochreiter:1997:LSM:1246443.1246450,Schuster:1997:BRN:2198065.2205129}, Multi-view LSTMs~\cite{rajagopalan2016extending}, and tensor product recurrent models (TFN)~\cite{tensoremnlp17}. For multimodal personality traits recognition on POM (left side of Table~\ref{all}), our baseline is able to additionally outperform more complicated memory-based recurrent models such as MFN~\citep{zadeh2018memory} on several metrics. We view this as an impressive achievement considering the simplicity of our model and the significantly fewer parameters that our model contains. As we will later show, our model's strong performance comes with the additional benefit of being significantly faster than the existing models.

\subsection{Semi-supervised Learning}

\newcolumntype{K}[1]{>{\centering\arraybackslash}p{#1}}
\definecolor{gg}{RGB}{45,190,45}

\begin{table}[t!]
\fontsize{7.5}{10}\selectfont
\centering
\setlength\tabcolsep{5.0pt}
\begin{tabular}{c : l : *{16}{K{1.05cm}}}
\Xhline{3\arrayrulewidth}
\multirow{ 3}{*}{\% Labels} & Dataset & \multicolumn{2}{c}{\textbf{CMU-MOSI}} \\
& Task & \multicolumn{2}{c}{\textbf{Sentiment}} \\
& Metric       & A$({2})$ & F1 \\
\Xhline{0.5\arrayrulewidth}
\multirow{ 3}{*}{40\%}  & AE    & 55.4 & 54.7 \\
& seq2seq               & 56.4 & 49.3 \\
& \textbf{MMB2}         & \textbf{72.9} & \textbf{72.8} \\
\Xhline{0.5\arrayrulewidth}
\multirow{ 3}{*}{60\%}  & AE    & 55.2 & 54.2 \\
& seq2seq               & 56.3 & 51.5 \\
& \textbf{MMB2}         & \textbf{73.6} & \textbf{73.5} \\
\Xhline{0.5\arrayrulewidth}
\multirow{ 3}{*}{80\%}  & AE    & 55.2 & 54.8 \\
& seq2seq               & 55.7 & 54.7 \\
& \textbf{MMB2}         & \textbf{74.1} & \textbf{74.1} \\
\Xhline{0.5\arrayrulewidth}
\multirow{ 3}{*}{100\%} & AE    & 55.2 & 53.2 \\
& seq2seq               & 57.0 & 54.1 \\
& \textbf{MMB2}         & \textbf{75.1} & \textbf{75.1} \\
\Xhline{3\arrayrulewidth}
\end{tabular}
\caption{Semi-supervised sentiment prediction results on CMU-MOSI. Our model outperforms deep autoencoders (AE) and their recurrent variant (seq2seq), remaining strong despite limited labeled data.}
\label{ssl}
\end{table}

Our next set of experiments evaluates the performance of our proposed baseline models when there is limited labeled data. Intuitively, we expect our model to have a lower sample complexity since training our model involves learning fewer parameters. As a result, we hypothesize that our model will generalize better when there is limited amounts of labeled data as compared to larger neural models with a greater number of parameters.

We test this hypothesis by evaluating the performance of our model on the CMU-MOSI dataset with only 40\%, 60\%, 80\%, and 100\% of the training labels. The remainder of the train set now consists of unlabeled data which is also used during training but in a semi-supervised fashion. We use the entire train set (both labeled and unlabeled data) for unsupervised learning of our multimodal embeddings before the embeddings are fine-tuned to predict the label using limited labeled data. A comparison is performed with two models that also learn embeddings from unlabeled multimodal utterances: 1) deep averaging autoencoder (\textbf{AE})~\citep{P15-1162,Hinton504} which averages the temporal dimension before using a fully connected autoencoder to learn a latent embedding, and 2) sequence to sequence autoencoder (\textbf{seq2seq})~\citep{sutskever2014sequence} which captures temporal information using a recurrent neural network encoder and decoder. For each of these models, an autoencoding model is used to learn embeddings on the entire training set (both labeled and unlabeled data) before the embeddings are fine-tuned to predict the label using limited labeled data. The validation and test sets remains unchanged for fair comparison.

Under this semi-supervised setting, we show prediction results on the CMU-MOSI test set in Table~\ref{ssl}. Empirically, we find that our model is able to outperform deep autoencoders and their recurrent variant. Our model remains strong and only suffers a drop in performance of about 3\% (75.1\% $\rightarrow$ 72.9\% binary accuracy) despite having access to only 40\% of the labeled training data.

\subsection{Inference Timing Comparisons}

To demonstrate another strength of our model, we compare the inference times of our model with existing baselines in Table~\ref{times}. Our model achieves an inference per second (IPS) of more than 10 times the closest neural model (EF-LSTM). We attribute this speedup to our (approximate) closed form solution for $m_\mathbf{s}$ as derived in Theorem~\ref{mle} and Corollary~\ref{mle2}, the small size of our model, as well as the fewer number of parameters (linear transformation parameters and classifier parameters) involved.

\begin{table}[t!]
\fontsize{7.5}{10}\selectfont
\centering
\setlength\tabcolsep{5.5pt}
\begin{tabular}{l : c c}
\Xhline{3\arrayrulewidth}
Method       & \textbf{Average Time (s)} & \textbf{Inferences Per Second (IPS)} \\ 
\Xhline{0.5\arrayrulewidth}
DF              & 0.305 & 1850 \\
EF-LSTM			& 0.022 & 31200\\
MV-LSTM			& 0.490 & 1400 \\
BC-LSTM         & 0.210 & 3270 \\
TFN             & 2.058 & 333  \\
MFN         	& 0.144 & 4760 \\
\Xhline{0.5\arrayrulewidth}
MMB1            & \textbf{0.00163} & \textbf{421000} \\
MMB2            & \textbf{0.00219} & \textbf{313000} \\
\Xhline{3\arrayrulewidth}
\end{tabular}
\caption{Average time taken for inference on the CMU-MOSI test set and Inferences Per Second (IPS) on a single Nvidia GeForce GTX 1080 Ti GPU, averaged over 5 trials. Our proposed baselines are more than 10 times faster than the closest neural model (EF-LSTM).}
\label{times}
\end{table}

\subsection{Ablation Study}

\newcolumntype{K}[1]{>{\centering\arraybackslash}p{#1}}
\definecolor{gg}{RGB}{15,125,15}
\definecolor{rr}{RGB}{190,45,45}

\begin{table}[t!]
\fontsize{7.5}{10}\selectfont
\centering
\setlength\tabcolsep{1.0pt}
\begin{tabular}{l : *{2}{K{0.5cm}} : *{16}{K{0.9cm}}}
\Xhline{3\arrayrulewidth}
Dataset & & & \multicolumn{2}{c}{\textbf{CMU-MOSI}} \\
Task & & & \multicolumn{2}{c}{\textbf{Sentiment}} \\
Model   & PE & FT & A$({2})$ & F1 \\
\Xhline{0.5\arrayrulewidth}
MMB2, language only & \textcolor{gg}\cmark & \textcolor{gg}\cmark & 72.3 & 73.7 \\
MMB2 & \textcolor{rr}\xmark & \textcolor{rr}\xmark & 74.1 & 73.9 \\ 
MMB2 & \textcolor{rr}\xmark & \textcolor{gg}\cmark & 74.6 & 74.6 \\
MMB2 & \textcolor{gg}\cmark & \textcolor{rr}\xmark & 74.6 & 74.6 \\
MMB2 & \textcolor{gg}\cmark & \textcolor{gg}\cmark & \textbf{75.2} & \textbf{75.1} \\
\Xhline{3\arrayrulewidth}
\end{tabular}
\caption{Ablation studies on CMU-MOSI test set. Incorporating nonverbal (visual and acoustic) features, positional encodings (PE), and task-specific fine tuning (FT) are important for good prediction performance.}
\label{ablation}
\end{table}

To further motivate our design decisions, we test some ablations of our model: 1) we remove the modeling capabilities of the visual and acoustic modalities, instead modeling only the language modality, 2) we remove the positional encodings, and 3) we remove the fine tuning step. We provide these results in Table~\ref{ablation} and observe that each component is indeed important for our model. Although the text only model performs decently, incorporating visual and acoustic features under our modeling assumption improves performance. Our results also demonstrate the effectiveness of positional encodings and fine tuning without having to incorporate any additional learnable parameters.
\section{Conclusion}

This paper proposed two simple but strong baselines to learn embeddings of multimodal utterances. The first baseline assumes a factorization of the utterance into unimodal factors conditioned on the joint embedding while the second baseline extends the first by assuming a factorization into unimodal, bimodal, and trimodal factors. Both proposed models retain simplicity and efficiency during both learning and inference. From experiments across multimodal tasks and datasets, we show that our proposed baseline models: 1) display competitive performance on supervised multimodal prediction, 2) outperform classical deep autoencoders for semi-supervised multimodal prediction and 3) attain significant (10 times) speedup during inference. Overall, we believe that our strong baseline models provide new benchmarks for future research in multimodal learning.

\section*{Acknowledgements}

PPL and LM were partially supported by Samsung and NSF (Award 1750439). Any opinions, findings, and conclusions or recommendations expressed in this material are those of the author(s) and do not necessarily reflect the views of Samsung and NSF, and no official endorsement should be inferred. YHT and RS were supported in part by the NSF IIS1763562, Office of Naval Research N000141812861, and Google focused award. We would also like to acknowledge NVIDIA's GPU support and the anonymous reviewers for their constructive comments on this paper.

\bibliography{naaclhlt2019}
\bibliographystyle{acl_natbib}

\clearpage

\appendix

\section{Appendix}

\subsection{Proof of Theorem 1}

We begin by restating the likelihood of a multimodal segment $\mathbf{s}$ under our model:
\begin{align}
& \ \ \ \ \mathbb{P} [\mathbf{s}|m_{\mathbf{s}}] \\
&= \mathbb{P} [\mathbf{w}|m_{\mathbf{s}}]^{\alpha_{\mathbf{w}}} \mathbb{P} [\mathbf{v}|m_{\mathbf{s}}]^{\alpha_{\mathbf{v}}} \mathbb{P} [\mathbf{a}|m_{\mathbf{s}}]^{\alpha_{\mathbf{a}}} \\
&= \prod_{w \in \mathbf{w}} \mathbb{P} [w|m_{\mathbf{s}}]^{\alpha_{\mathbf{w}}} \prod_{v \in \mathbf{v}} \mathbb{P} [v|m_{\mathbf{s}}]^{\alpha_{\mathbf{v}}} \prod_{a \in \mathbf{a}} \mathbb{P} [a|m_{\mathbf{s}}]^{\alpha_{\mathbf{a}}}
\end{align}

We define the objective function by the maximum likelihood estimator of the multimodal utterance embedding and the parameters. The estimator is obtained by solving the unknown variables that maximizes the log-likelihood of the observed multimodal utterance (i.e., $\mathbf{s}$):
\begin{align}
\label{objective_app}
& \mathcal{L} (m_{\mathbf{s}}, W^{}_{}, b^{}_{}; \mathbf{s}) = \log \mathbb{P} [\mathbf{s}|m_{\mathbf{s}}; W^{}_{}, b^{}_{}] \\
\nonumber &= \sum_{w \in \mathbf{w}} \log \mathbb{P} [w|m_{\mathbf{s}}]^{\alpha_{\mathbf{w}}} + \sum_{v \in \mathbf{v}} \log \mathbb{P} [v|m_{\mathbf{s}}]^{\alpha_{\mathbf{v}}} \\
&+ \sum_{a \in \mathbf{a}} \log \mathbb{P} [a|m_{\mathbf{s}}]^{\alpha_{\mathbf{a}}}
\end{align}
with $W$ and $b$ denoting all linear transformation parameters. Our goal is to solve for the optimal embedding $m_{\mathbf{s}}^* = \argmax_{m_{\mathbf{s}}^*} \mathcal{L} (m_{\mathbf{s}}, W^{}_{}, b^{}_{}; \mathbf{s})$. We will begin by simplifying each of the terms: $\log \left( \mathbb{P} [w|m_{\mathbf{s}}] \right)^{\alpha_{\mathbf{w}}}, \log \left( \mathbb{P} [v|m_{\mathbf{s}}] \right)^{\alpha_{\mathbf{v}}}$, and $\log \left( \mathbb{P} [a|m_{\mathbf{s}}] \right)^{\alpha_{\mathbf{a}}}$.

For the language features, we follow the approach in~\cite{arora2}. We define:
\begin{align}
& \ \ \ \ f_{w}(m_{\mathbf{s}}) \\
&= \log \mathbb{P} [w|m_{\mathbf{s}}]^{\alpha_{\mathbf{w}}} \\
&= {\alpha_{\mathbf{w}}} \log \mathbb{P} [w|m_{\mathbf{s}}] \\
&= {\alpha_{\mathbf{w}}} \log \left[ \alpha p(w) + (1-\alpha) \frac{\exp \left( \langle w , m_{\mathbf{s}} \rangle \right)}{Z_{m_{\mathbf{s}}}} \right]
\end{align}
By taking the gradient $\nabla_{m_{\mathbf{s}}} f_{w}(m_{\mathbf{s}})$ and making a Taylor approximation,
\begin{align}
& \ \ \ \ f_{w}(m_{\mathbf{s}}) \approx f_{w}(0) + \nabla_{m_{\mathbf{s}}} f_{w}(0)^\top m_{\mathbf{s}} \\
&= c + \frac{{\alpha_{\mathbf{w}}} (1-\alpha)/(\alpha Z)}{p(w)+(1-\alpha)/(\alpha Z)} \langle w, m_{\mathbf{s}} \rangle
\end{align}
For the visual features, we can decompose the likelihood $\mathbb{P} [v|m_{\mathbf{s}}]$ as a product of the likelihoods in each coordinate $\prod_{i=1}^{|v|} \mathbb{P} [v(i)|m_{\mathbf{s}}]$ since we assume a diagonal covariance matrix. Let $v(i) \in \mathbb{R}$ denote the $i$th visual feature and $W_v^\mu(i) \in \mathbb{R}^{|m_{\mathbf{s}}|}$ be the $i$-th column of $W_v^\mu$.
\begin{align}
\mu_v(i) &= W_v^\mu(i) m_{\mathbf{s}} + b_v^\mu(i) \\
\sigma_v(i) &= \exp \left( W_v^\sigma(i) m_{\mathbf{s}} + b_v^\sigma(i) \right) \\
v(i) | m_{\mathbf{s}} &\sim N(\mu_v(i), \sigma_v^2(i)) \\
\mathbb{P} [v(i)|m_{\mathbf{s}}] &= \frac{1}{\sqrt{2\pi}\sigma_v(i)} \exp \left( - \frac{(v(i)-\mu_v(i))^2}{2\sigma_v^2(i)} \right)
\end{align}

Define $f_{v(i)}(m_{\mathbf{s}})$ as follows:
\begin{align}
& \ \ \ \ f_{v(i)}(m_{\mathbf{s}}) \\
&= \log \mathbb{P} [v(i)|m_{\mathbf{s}}]^{\alpha_{\mathbf{v}}} \\
&= {\alpha_{\mathbf{v}}} \log \mathbb{P} [v(i)|m_{\mathbf{s}}] \\
&= - {\alpha_{\mathbf{v}}} \log \left( \sqrt{2\pi} \sigma_v(i) \right) - {\alpha_{\mathbf{v}}} \frac{(v(i)-\mu_v(i))^2}{2\sigma_v^2(i)} \\
\nonumber &= - {\alpha_{\mathbf{v}}} \log \left( \sqrt{2\pi} \exp \left( W_v^\sigma(i) m + b_v^\sigma(i) \right) \right) \\
& \ \ \ - {\alpha_{\mathbf{v}}} \frac{(v(i)-W_v^\mu(i) m - b_v^\mu(i))^2}{2\exp \left( W_v^\sigma(i) m_{\mathbf{s}} + b_v^\sigma(i) \right)^2} \\
\nonumber &= - {\alpha_{\mathbf{v}}} \log \sqrt{2\pi} - \left( W_v^\sigma(i) m_{\mathbf{s}} + b_v^\sigma(i) \right) \\
& \ \ \ - {\alpha_{\mathbf{v}}} \frac{(v(i)-W_v^\mu(i) m - b_v^\mu(i))^2}{2\exp \left( 2 W_v^\sigma(i) m_{\mathbf{s}} + 2 b_v^\sigma(i) \right)}
\end{align}

The gradient $\nabla_{m_{\mathbf{s}}} f_{v(i)}(m_{\mathbf{s}})$ is as follows
\begin{align}
& \ \ \ \ \nabla_{m_{\mathbf{s}}} f_{v(i)}(m_{\mathbf{s}}) \\
&= - {\alpha_{\mathbf{v}}}W_v^\sigma(i) - {\alpha_{\mathbf{v}}} \frac{1}{4\sigma_v(i)^4} \left[ 2 (v(i)-\mu_v(i)) \right] \\
\nonumber &= {\alpha_{\mathbf{v}}} \frac{\left[ (v(i)-\mu_v(i)) W_v^\mu(i) + (v(i)-\mu_v(i))^2 W_v^\sigma(i) \right]}{\sigma_v(i)^2} \\
& \ \ \ - {\alpha_{\mathbf{v}}} W_v^\sigma(i) \\
\nonumber &= {\alpha_{\mathbf{v}}} \frac{v(i)-\mu_v(i)}{\sigma_v(i)^2} W_v^\mu(i) \\
& \ \ \ + {\alpha_{\mathbf{v}}} \left(\frac{(v(i)-\mu_v(i))^2}{\sigma_v(i)^2}-1\right) W_v^\sigma(i)
\end{align}

By Taylor expansion, we have that
\begin{align}
& \ \ \ \ f_{v(i)}(m_{\mathbf{s}}) \\
&\approx f_{v(i)}(0) + \nabla_{m_{\mathbf{s}}} f_{v(i)}(0)^\top m_{\mathbf{s}} \\
\nonumber &= \underbrace{- {\alpha_{\mathbf{v}}}\log \left( \sqrt{2\pi} \exp \left( b_v^\sigma(i) \right) \right) - {\alpha_{\mathbf{v}}} \frac{(v(i) - b_v^\mu(i))^2}{2 \exp \left( 2 b_v^\sigma(i) \right)}}_\text{constant with respect to $m_{\mathbf{s}}$} \\
\nonumber & \ \ \ + {\alpha_{\mathbf{v}}} \frac{v(i) - b_v^\mu(i)}{\exp \left(2 b_v^\sigma(i) \right)} \langle W_v^\mu(i), m_{\mathbf{s}} \rangle \\
& \ \ \ + {\alpha_{\mathbf{v}}} \left(\frac{(v(i)-b_v^\mu(i))^2}{\exp \left(2b^\sigma_v(i) \right)}-1\right) \langle W_v^\sigma(i), m_{\mathbf{s}} \rangle \\
\nonumber &= c + {\alpha_{\mathbf{v}}} \frac{v(i) - b_v^\mu(i)}{\exp \left(2 b_v^\sigma(i) \right)} \langle W_v^\mu(i), m_{\mathbf{s}} \rangle \\
& \ \ \ + {\alpha_{\mathbf{v}}} \left(\frac{(v(i)-b_v^\mu(i))^2}{\exp \left(2b^\sigma_v(i) \right)}-1\right) \langle W_v^\sigma(i), m_{\mathbf{s}} \rangle
\end{align}

By our symmetric paramterization of the acoustic features, we have that:
\begin{align}
& \ \ \ \ f_{a(i)}(m_{\mathbf{s}}) \\
\nonumber &\approx c + {\alpha_{\mathbf{a}}} \frac{a(i) - b_a^\mu(i)}{\exp \left(2 b_a^\sigma(i) \right)} \langle W_a^\mu(i), m_{\mathbf{s}} \rangle \\
& \ \ \ + {\alpha_{\mathbf{a}}} \left(\frac{(a(i)-b_a^\mu(i))^2}{\exp \left(2b^\sigma_a(i) \right)}-1\right) \langle W_a^\sigma(i), m_{\mathbf{s}} \rangle
\end{align}

Rewriting this in matrix form, we obtain that
\begin{align}
f_{w}(m_{\mathbf{s}}) = c + \psi_w \langle w, m_{\mathbf{s}} \rangle
\end{align}
\begin{align}
f_{v}(m_{\mathbf{s}}) &= \sum_{i \in |v|} f_{v(i)}(m_{\mathbf{s}}) \\
\nonumber &= c + \left\langle W_v^{\mu\top} (v-b_v^\mu) \psi_v^{(1)}, m_{\mathbf{s}} \right\rangle \\
& \ \ \ \ + \left\langle W_v^{\sigma\top} (v-b_v^\mu) \otimes (v-b_v^\mu) \psi_v^{(2)}, m_{\mathbf{s}} \right\rangle
\end{align}
\begin{align}
f_{a}(m_{\mathbf{s}}) &= \sum_{i \in |a|} f_{a(i)}(m_{\mathbf{s}}) \\
\nonumber &= c + \left\langle W_a^{\mu\top} (a-b_a^\mu) \psi_a^{(1)}, m_{\mathbf{s}} \right \rangle \\
& \ \ \ \ + \left\langle W_a^{\sigma\top} (a-b_a^\mu) \otimes (a-b_a^\mu) \psi_a^{(2)}, m_{\mathbf{s}} \right\rangle
\end{align}
where $\otimes$ denotes Hadamard (element-wise) product and the weights $\psi$'s are given as follows:
\begin{align}
\psi_w &= \frac{{\alpha_\mathbf{w}} (1-\alpha)/(\alpha Z)}{p(w)+(1-\alpha)/(\alpha Z)} \\
\psi_v^{(1)} &= \textrm{diag} \left( \frac{\alpha_{\mathbf{v}}}{\exp \left(2 b_v^\sigma \right)} \right) \\
\psi_v^{(2)} &= \textrm{diag} \left(\frac{\alpha_{\mathbf{v}}}{\exp \left(2b^\sigma_v \right)}-\alpha_{\mathbf{v}}\right) \\
\psi_a^{(1)} &= \textrm{diag} \left( \frac{\alpha_{\mathbf{a}}}{\exp \left(2 b_a^\sigma \right)} \right) \\
\psi_a^{(2)} &= \textrm{diag} \left(\frac{\alpha_{\mathbf{a}}}{\exp \left(2b^\sigma_a \right)}-\alpha_{\mathbf{a}}\right)
\end{align}
Observe that $W_v^{\sigma\top} (v-b_v^\mu)$ is a composition of a shift $-b_v^\mu$ and a linear transformation $W_v^{\sigma\top}$ of the visual features into the multimodal embedding space. Note that $\mathbb{E} [v|m_{\mathbf{s}}] = b_v^\mu$. In other words, this shifts the visual features towards 0 in expectation before transforming them into the multimodal embedding space. Our choice of a Gaussian likelihood for the visual and acoustic features introduces a squared term $W_v^{\sigma\top} (v-b_v^\mu) \otimes (v-b_v^\mu)$ to account for the $\ell_2$ distance present in the Gaussian pdf. Secondly, regarding the weights $\psi$'s, note that: 1) the weights for a modality are proportional to the global hyper-parameters $\alpha$ assigned to that modality, and 2) the weights $\psi_w$ are inversely proportional to $p(w)$ (rare words carry more weight). The weights $\psi_v$'s and $\psi_a$'s scales each feature dimension inversely by their magnitude.

Finally, we know that our objective function~\eqref{objective_app} decomposes as
\begin{align}
\nonumber & \ \ \ \ \mathcal{L} (m_{\mathbf{s}}, W^{}_{}, b^{}_{}; \mathbf{s}) \\
&= \sum_{w \in \mathbf{w}} f_{w}(m_{\mathbf{s}}) + \sum_{v \in \mathbf{v}} f_{v}(m_{\mathbf{s}}) + \sum_{a \in \mathbf{a}} f_{a}(m_{\mathbf{s}})
\end{align}

We now use the fact that $\max_{x: \|x\|_2^2 =1} \textrm{constant } + \langle x, g \rangle = g / \| g \|$. If we assume that $m_{\mathbf{s}}^*$ lies on the unit sphere, the maximum likelihood estimate for $m_{\mathbf{s}}$ is approximately:
\begin{align}
&\nonumber \ \ \ \ m_{\mathbf{s}}^* \\
\nonumber &= \sum_{w \in \mathbf{w}} \psi_w w + \sum_{v \in \mathbf{v}} \left( W_v^{\mu\top} \tilde{v}^{(1)} \psi_v^{(1)} + W_v^{\sigma\top} \tilde{v}^{(2)} \psi_v^{(2)} \right)\\
&+ \sum_{a \in \mathbf{a}} \left( W_a^{\mu\top} \tilde{a}^{(1)} \psi_a^{(1)} + W_a^{\sigma\top} \tilde{a}^{(2)} \psi_a^{(2)} \right).
\end{align}
where we have rewritten the shifted (and squared) visual and acoustic terms as
\begin{align}
\tilde{v}^{(1)} &= v-b_v^\mu \\
\tilde{v}^{(2)} &= (v-b_v^\mu) \otimes (v-b_v^\mu) \\
\tilde{a}^{(1)} &= a-b_a^\mu \\
\tilde{a}^{(2)} &= (a-b_a^\mu) \otimes (a-b_a^\mu)
\end{align}
which concludes the proof.

\subsection{Multimodal Features}

Here we present extra details on feature extraction for the language, visual and acoustic modalities.

\noindent \textbf{Language:} We used 300 dimensional GloVe word embeddings trained on 840 billion tokens from the Common Crawl dataset \cite{pennington2014glove}. These word embeddings were used to embed a sequence of individual words from video segment transcripts into a sequence of word vectors that represent spoken text. 

\noindent \textbf{Visual:} The library Facet \cite{emotient} is used to extract a set of visual features including facial action units, facial landmarks, head pose, gaze tracking and HOG features \cite{zhu2006fast}. These visual features are extracted from the full video segment at 30Hz to form a sequence of facial gesture measures throughout time.

\noindent \textbf{Acoustic:} The software COVAREP \cite{degottex2014covarep} is used to extract acoustic features including 12 Mel-frequency cepstral coefficients, pitch tracking and voiced/unvoiced segmenting features \cite{drugman2011joint}, glottal source parameters \cite{childers1991vocal,drugman2012detection,alku1992glottal,alku1997parabolic,alku2002normalized}, peak slope parameters and maxima dispersion quotients \cite{kane2013wavelet}. These visual features are extracted from the full audio clip of each segment at 100Hz to form a sequence that represent variations in tone of voice over an audio segment.

\subsection{Multimodal Alignment}
We perform forced alignment using P2FA \cite{P2FA} to obtain the exact utterance time-stamp of each word. This allows us to align the three modalities together. Since words are considered the basic units of language we use the interval duration of each word utterance as one time-step. We acquire the aligned video and audio features by computing the expectation of their modality feature values over the word utterance time interval \cite{zadeh2018memory}.

\end{document}